\documentclass[nohyperref]{article}

\usepackage{microtype}
\usepackage{graphicx}
\usepackage{booktabs} 

\usepackage{hyperref}

\usepackage[accepted]{icml2022}

\usepackage{amsmath}
\usepackage{amssymb}
\usepackage{mathtools}
\usepackage{amsthm}

\usepackage[capitalize,noabbrev]{cleveref}

\theoremstyle{plain}
\newtheorem{theorem}{Theorem}[section]

\newtheorem{lemma}[theorem]{Lemma}

\theoremstyle{definition}

\theoremstyle{remark}

\usepackage[textsize=tiny]{todonotes}

\icmltitlerunning{Beyond Ridge Regression for Distribution-Free Data}

\usepackage{amsmath,amsfonts,bm}

\def\thetay{\hat{\theta}_{y}}
\def\thetaytag{\hat{\theta}_{y'}}
\def\probthetay{p_{\thetay}(y|x)}
\def\probthetaytag{p_{\thetaytag}(y'|x)}

\def\thetagenie{\hat{\theta}(\mathcal{D}_N;x,y)}
\def\probthetagenie{p_{\thetagenie}(y|x)}
\def\thetagenietag{\hat{\theta}(\mathcal{D}_N;x,y')}
\def\probthetagenietag{p_{\thetagenietag}(y'|x)}

\def\thetahat{\hat{\theta}}
\def\thetalamb{\hat{\theta}_\lambda}

\newcommand{\bnorm}[1]{\big\lVert#1\big\rVert}
\newcommand{\norm}[1]{||#1||}

\def\Tableref#1{Table~\ref{#1}}

\def\DN{\mathcal{D}_N} 
\def\Plamb{P_\lambda}
\def\Klamb{K_\lambda}

\def\Klpnml{K_\textit{LpNML}}
\def\Klpnml2{K_\textit{LpNML2}}
\def\mulpnml{\hat{\mu}_\textit{\tiny{LpNML}}}
\def\sigmalpnml{\hat{\sigma}^2_\textit{\tiny{LpNML}}}

\def\figref#1{figure~\ref{#1}}
\def\Figref#1{Figure~\ref{#1}}

\def\secref#1{section~\ref{#1}}

\def\eqref#1{(\ref{#1})}

\def\1{\bm{1}}

\DeclareMathAlphabet{\mathsfit}{\encodingdefault}{\sfdefault}{m}{sl}
\SetMathAlphabet{\mathsfit}{bold}{\encodingdefault}{\sfdefault}{bx}{n}

\DeclareMathOperator*{\argmin}{arg\,min}

\usepackage{url}
\usepackage{multirow}
\usepackage{wrapfig}
\usepackage{placeins}
\usepackage{array}
\usepackage{lipsum} 
\usepackage{enumitem}
\usepackage{caption}        
\usepackage{subcaption} 
\usepackage{array, booktabs, makecell}
\usepackage{afterpage}

\newcommand{\ignore}[1]{}

\definecolor{green}{rgb}{0.05, 0.5, 0.06}
\definecolor{pink}{rgb}{0.91, 0.33, 0.5}

\usepackage{hyperref}

\begin{document}

\twocolumn[
\icmltitle{Beyond Ridge Regression for Distribution-Free Data}


\icmlsetsymbol{equal}{*}

\begin{icmlauthorlist}
\icmlauthor{Koby Bibas}{tau}
\icmlauthor{Meir Feder}{tau}
\end{icmlauthorlist}

\icmlaffiliation{tau}{Tel-Aviv University, Tel Aviv, Israel}
\icmlcorrespondingauthor{Koby Bibas}{kobybibas@gmail.com}

\icmlkeywords{Machine Learning, Linear Regression, Predictive Normalized Maximum Likelihood, Regret, Luckiness}

\vskip 0.3in
]



\printAffiliationsAndNotice{}  

\begin{abstract}
In supervised batch learning, the predictive normalized maximum likelihood (pNML) has been proposed as the min-max regret solution for the distribution-free setting, where no distributional assumptions are made on the data.
However, the pNML is not defined for a large capacity hypothesis class as over-parameterized linear regression.
For a large class, a common approach is to use regularization or a model prior.
In the context of online prediction where the min-max solution is the Normalized Maximum Likelihood (NML), it has been suggested to use NML with ``luckiness'': A prior-like function is applied to the hypothesis class, which reduces its effective size.
Motivated by the luckiness concept, for linear regression we incorporate a luckiness function that penalizes the hypothesis proportionally to its $\ell_2$ norm.
This leads to the ridge regression solution.
The associated pNML with luckiness (LpNML) prediction deviates from the ridge regression empirical risk minimizer (Ridge ERM): When the test data reside in the subspace corresponding to the small eigenvalues of the empirical correlation matrix of the training data, the prediction is shifted toward 0.
Our LpNML reduces the Ridge ERM error by up to 20\% for the PMLB sets, and is up to 4.9\% more robust in the presence of distribution shift compared to recent leading methods for UCI sets.
\end{abstract}

\begin{figure*}[t]
\centering
\begin{subfigure}{0.49\textwidth}
\centering
    \includegraphics[width=\textwidth]{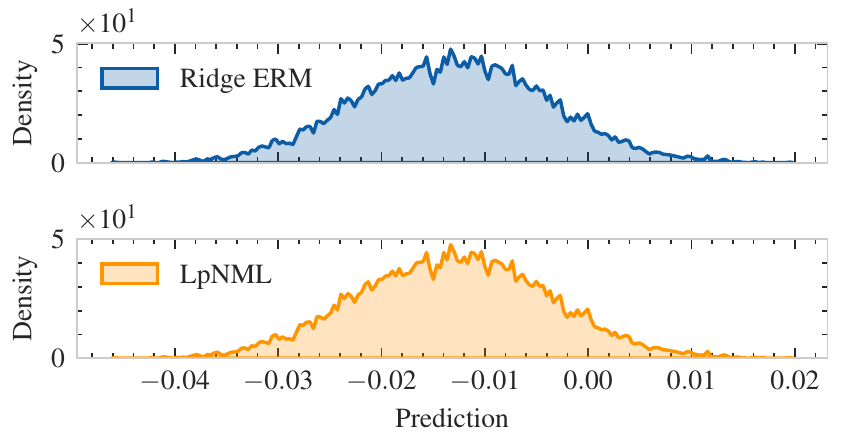}
    \vspace{-0.5cm}
    \caption{Greater than 0 eigenvalue subspace}
    \label{fig:lpnml_parallel}
\end{subfigure}
\begin{subfigure}{0.49\textwidth}
\centering
    \includegraphics[width=\textwidth]{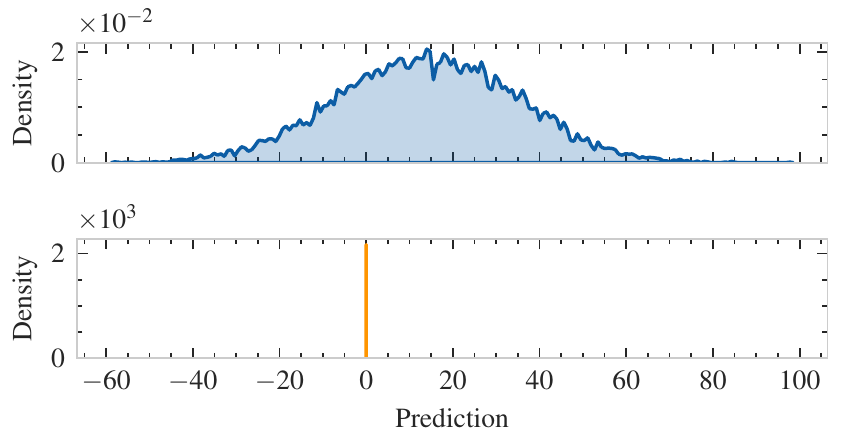}
    \vspace{-0.5cm}
    \caption{Equal to 0 eigenvalue subspace}
    \label{fig:lpnml_orthogonal}
\end{subfigure}
\caption{\textbf{Ridge ERM and LpNML prediction histograms.} Over-parameterized linear regression predictions of test samples that reside in the subspace that is spanned by the eigenvectors of the training data empirical correlation matrix corresponding to the eigenvalues that are (a) greater than 0, or (b) equal 0. In scenario (a), the LpNML prediction equals the Ridge ERM. On the other hand, in scenario (b), the LpNML prediction equals 0. See~\secref{sec:exp_lpnml_prediction} for more information. \label{fig:lpnml_prediction}}
\end{figure*}

\section{Introduction} \label{Introduction}
Ridge regression is a widely used method for linear regression when the data dimension is large compared to the training set size.
It has been applied in a large variety of domains such as econometrics~\cite{sengupta2020asymptotic}, bioinformatics~\cite{xu2020blood}, and social science~\cite{grimmer2021machine}.
From a Bayesian perspective, it coincides with the mean of the predictive distribution where the parameter prior and noise are Gaussian~\cite{DBLP:conf/iclr/LiuD20}.

The most popular variant is the \textit{Ridge empirical risk minimizer} (Ridge ERM): the model is chosen to minimize the training set loss and the ridge parameter is selected either to minimize a validation set or with the leave-one-out protocol, which leads to the same asymptotic performance as the optimally-tuned ridge estimator~\cite{hastie2019surprises}.

When using Ridge ERM, the underlining assumption is that there is a probabilistic relationship between the data and labels and between the training and test. 
In the \textit{stochastic setting}, see \citet{merhav1998universal}, it is assumed that the probabilistic relation between the test feature $x$ and its label $y$ is given by an (unknown) model from a given hypothesis class $P_\Theta$. 
For \textit{the probably approximately correct} (PAC) setting~\cite{valiant1984theory}, $x$ and $y$ are assumed to be generated by some source $P(x,y)=P(x)P(y|x)$ which is not necessarily a member of the hypothesis class.
These assumptions, however, may not hold in a real-world scenario.

We consider the \textit{individual setting}~\citep{merhav1998universal} which is considered as the most general framework. In this setting, there is no assumption on a probabilistic mechanism that generates the training and test data.
Moreover, the relation between the data and labels can be deterministic and may even be determined by an adversary.

In the individual setting, the goal is to seek a learner that can compete with a reference learner, a \textit{genie}, with the following properties~\cite{fogel2019universal}: (i) knows the true test label value, (ii) is restricted to use a model from a given hypothesis class $P_\Theta$, and (iii) does not know which of the samples is the test. 
This genie chooses the model that minimizes the loss for the training set $\DN=\left\{(x_n,y_n)\right\}_{n=1}^N$ and the specific test sample $(x,y)$
\begin{equation}
\thetagenie  = \arg\min_{\theta \in \Theta} \bigg[ \sum_{n=1}^N \ell(p_\theta;x_n,y_n) + \ell(p_\theta;x,y) \bigg].
\end{equation}
When using the log-loss to evaluate the hypothesis performance $\ell(p_\theta,x,y) = - \log p_\theta(y|x)$, the \textit{regret} is defined as the loss difference between a learner $q$ and the genie
\begin{equation} \label{eq:genie_regret}
R(\DN,x,y,q) = \log \frac{\probthetagenie}{q(y|x)}.
\end{equation}
The \textit{predictive normalized maximum likelihood} (pNML) learner minimizes the regret for the worst-case test label 
\begin{equation} \label{eq:minmax_prob}
\Gamma = \min_q \max_y R(\DN,x,y,q).
\end{equation}
and its predictive distribution is~\cite{fogel2019universal}
\begin{equation} \label{eq:pNML}
q_{\mbox{\tiny{pNML}}}(y|x)=\frac{\probthetagenie}{\int \probthetagenietag dy'} .
\end{equation}
However, the pNML may not be defined for an \textit{over-parameterized} hypothesis class, where the number of parameters exceeds the training set size. The reason is that in the denominator of~\eqref{eq:pNML}, every possible value of the test label $y'$ can be perfectly fitted such that the integral diverges.

The pNML root lies in the \textit{normalized maximum likelihood} (NML) approach for online prediction~\cite{shtar1987universal}.
Since the NML may also be improper, a leading solution is \textit{NML with luckiness} (LNML)~\cite{roos2004mdl}: A luckiness function $w(\theta)$ is designed such that on sequences with small $w(\theta)$, we are prepared to incur large regret~\cite{grunwald2007minimum}. 
This is the equivalent to a Bayesian model prior.

In this paper, we apply the luckiness concept to the pNML and call it LpNML.
For linear regression, we design the luckiness function to be proportional to the $\ell_2$ model norm.
This leads to a genie that equals ridge regression.
We derive the corresponding LpNML and show its \textit{prediction}, i.e., the mean of the predictive distribution, differs from the Ridge ERM's:
When the test sample lies within the subspace associated with the small eigenvalues of the empirical correlation matrix of the training data, the prediction is shifted toward 0. This behavior is shown in~\figref{fig:lpnml_prediction}.

To summarize, we make the following contributions.
\begin{itemize}
    \item We introduce the LpNML: This is the first use of the luckiness concept with the pNML learner.
    \item We analytically derive the LpNML for linear regression with $\ell_2$ regularization.
    \item We analyze the difference between the LpNML and Ridge ERM prediction, showing that the LpNML prediction is shifted to 0 when the test sample is in a different subspace than the training data.
\end{itemize}
We demonstrate the LpNML attains a better mean square error (MSE) than Ridge ERM for 50 real-world PMLB sets~\cite{Olson2017PMLB}, reducing the error by up to 20\%. Furthermore, we show the LpNML outperforms leading methods for the distribution-shift benchmark~\cite{tripuraneni2020single}, when the test set differs from training.

\section{Related work} \label{sec:related_work}

\paragraph{The pNML.}
\citet{roos2008sequentially} were the first to derive the pNML to accelerate the NML calculation for sequential prediction. Later, \citet{fogel2019universal} have shown that this pNML learner attains the min-max regret for supervised batch learning with individual data.

The pNML was estimated for deep-neural-networks and its min-max regret was used as a confidence measure to detect out-of-distribution inputs~\cite{bibas2019deep,DBLP:conf/icml/ZhouL21,bibas2021single}.
\citet{pesso2021utilizing} altered the input image based on the pNML procedure to defend against adversarial attacks.


\citet{bibas2019new} have derived the pNML for \textit{under-parameterized} linear regression, where the number of parameters is smaller than the training set size. However, when using a large capacity hypothesis set as the over-parameterized linear regression we use in this paper, the pNML regret becomes infinite and cannot be used.
\citet{bibas2021distribution} tackled this issue by constraining the norm of hypotheses to be equal to the minimum norm solution. This paper takes a different approach by incorporating a prior function, which favors solutions with small norms.

\paragraph{The LNML.}
For online prediction with individual sequences, the min-max optimal regret is given by the NML~\cite{shtar1987universal,grunwald2007minimum}.
The prediction is performed for the entire sequence: Denote the sequence $y^N = \left\{y_n \right\}_{n=1}^N$, the NML probability assignment is 
\begin{equation} \label{eq:nml}
q_\textit{\tiny{NML}}(y^N) = \frac{\max_\theta p_\theta(y^N)}{\int \max_\theta p_\theta (y'^N) dy'^N}.
\end{equation}

Since NML may be improper, several treatments have been proposed. Among these treatments is setting a restriction on the range of data or the range of parameters~\cite{hirai2011efficient}.
The drawback of this method is that samples can fall outside of any valid restrictions.
A different approach is the LNML that is also named ``generalized NML''~\cite{roos2004mdl}:
A luckiness function $w(\theta)$ is set such that the sequence distribution becomes
\begin{equation} \label{eq:lnml}
q_\textit{\tiny{LNML}}(y^N) = \frac{ \max_\theta  p_\theta(y^N) w(\theta)}{\int \max_\theta p_\theta (y^N) w(\theta) dy^N}.
\end{equation}
The advantage of LNML is that there is large freedom in choosing the luckiness function: We can choose a function that has particularly pleasant properties~\cite{grunwald2007minimum}.
Choosing the luckiness function to equal a constant for example, reduces the LNML back to the NML.




\citet{miyaguchi2017normalized} derived the LNML for multivariate normal distributions with the conjugate prior luckiness function.
\citet{dwivedi2021revisiting} incorporated a luckiness to the LNML that is proportional to the model norm to find the best ridge regularization factor $\lambda$. 
As opposed to our approach, once $\lambda$ was found, the prediction equals the Ridge ERM.

\paragraph{Transductive prediction.}
A transductive inference uses unlabeled data to improve predictions of unlabeled examples. \citet{chapelle2000transductive} chose the test label values to minimize the leave-one-out error of ridge regression with both training and test data. \citet{cortes2007transductive} estimated the label of the unlabeled test data $x$ by using only the labeled neighbors of $x$. They also presented error bounds for the VC-dimension. \citet{alquier2012transductive} established risk bounds for a transductive version of the Lasso learner. \citet{lei2021near} developed a min-max linear estimator in the presence of a covariate shift, where the maximum is for the target domain learnable parameters.

Although all of the above have demonstrated empirical and theoretical benefits, many unlabeled test points need to be simultaneously available.
\citet{tripuraneni2020single} presented a single point transductive procedure for linear regression that improves the prediction root mean square deviation (RMSE), especially under a distribution shift.

\section{Preliminaries} \label{sec:preliminaries}
\paragraph{Linear regression setting.}
Given a training set
\begin{equation}
\DN = \{(x_n, y_n)\}_{n=1}^{N} , \quad x_n \in R^{M \times 1}, \quad y_n \in R,
\end{equation}
the goal is to predict the test label $y$ based on a new data sample $x$.
A common assumption is a linear relationship between the data and labels with an additive white noise
\begin{equation} \label{eq:linear}
    y_n = \theta^\top x_n + e_n, \qquad e_n \sim \mathcal{N}(0,\sigma^2).
\end{equation}
Denote the design matrix and the training set label vector
\begin{equation} \label{eq:trainset_matrix}
X_N = \begin{bmatrix} x_1 & \dots & x_N \end{bmatrix}^\top,  \quad 
Y_N = \begin{bmatrix} y_1 & \dots & y_N  \end{bmatrix}^\top
\end{equation}
where $X_N \in R^{N \times M} $ and $Y_N  \in R^{N \times 1}$,
the ERM solution which minimizes the log-loss of the training set (and is also the maximum likelihood estimator) is
\begin{equation} \label{eq:linear_regression_erm}
\hat{\theta} = (X_N^\top X_N)^{-1} X_N^\top Y_N.
\end{equation}

\paragraph{Ridge ERM.}
In regularized linear regression, leading to the Ridge ERM learner, the hypothesis class $\Theta$ is a sphere $\norm{\theta}^2_2\leq A$~\cite{ridgeregression}. 
Expressing this constraint with the Lagrangian:
\begin{equation}
\mathcal{L}(\theta, \lambda)= \norm{Y_N - X_N \theta}^2_2 + \lambda \left( \norm{\theta}^2_2 - A \right),
\end{equation}
the Ridge ERM learnable vector is
\begin{equation} \label{eq:ridge_erm}
\thetalamb = \left(X_N^\top X_N + \lambda I \right)^{-1} X_N^\top Y_N.
\end{equation}

\paragraph{Bayesian linear regression.} %
Defining the Gaussian prior of the learnable parameters 
\begin{equation}
p(\theta) = \mathcal{N}\left(m_0,S_0 \right),
\end{equation}
the posterior distribution is (see \citet{deisenroth2020mathematics})
\begin{equation}
\begin{split}
    &p(\theta|D_N) = \mathcal{N} \left(\theta^\top m_N, S_N\right), \\ 
    & S_N = \big(S_0^{-1} + \sigma^{-2} X_N^\top X_N \big)^{-1}, \\ 
    & m_N = S_N \big(S_0^{-1} m_0 + \sigma^{-2} X_N^\top Y_N\big).
\end{split}
\end{equation}
To compute the predictive distribution with the posterior distribution, all plausible parameters are averaged according to the prior distribution. Setting $m_0=0$ and $S_0 = \sigma^{-2} \lambda I$, the predictive distribution is
\begin{equation}
\begin{split}
& q_\textit{\tiny{Bayesian}}(y|x) = \int p(y|x,\theta) p(\theta|\DN) d\theta.
\\ &=
\mathcal{N}\left(\thetalamb^\top x,\sigma^2\left[1+ x^\top \left(X_N^\top X_N + \lambda I \right)^{-1}  x \right] \right).
\end{split}
\end{equation}
The mean of the Bayesian learner predictive distribution equals Ridge ERM of~\eqref{eq:ridge_erm}.
The prediction uncertainty is encoded in the variance, and it is negatively correlated with the projection of the test sample on the design matrix.

\section{pNML with luckiness}
Inspired by the luckiness concept for NML, we define the genie, a learner that knows the true test label, with a luckiness function $w(\theta)$ as follows
\begin{equation}
\hat{\theta}_y = \argmin_{\theta \in \Theta} \bigg[ \sum_{n=1}^N \ell(p_\theta,x_n,y_n) + \ell(p_\theta,x,y) - \log w\left(\theta\right)  \bigg].
\end{equation}
The luckiness function is used as a model prior: The genie is more likely to select $\theta$ that yields a larger $w(\theta)$.
The related regret in this setting is
\begin{equation}
R(q,D_N,x,y,w) =  \frac{\log p_{\hat{\theta}_y} (y|x) w(\hat{\theta}_y)}{\log q(y|x)}.
\end{equation}

\begin{theorem}
The LpNML is the learner that minimizes the worst-case regret objective
\begin{equation}
   q_\textit{\tiny{LpNML}}(y|x) = \argmin_q \max_y R(q,D_N,x,y,w).
\end{equation}
The LpNML predictive distribution is
\begin{equation} \label{eq:lpnml}
q_\textit{\tiny{LpNML}}(y|x) = \frac{\probthetay w(\thetay)}{\int \probthetaytag w(\thetaytag) dy'}
\end{equation}
and its min-max regret is
\begin{equation}
\Gamma = \log \int \probthetaytag w(\thetaytag) dy'.
\end{equation}
\end{theorem}
\begin{proof}
This proof essentially follows that of~\citet{fogel2019universal} with the additional luckiness function:
The LpNML has a valid predictive distribution $\int q_\textit{LpNML}(y|x) dy = 1$.
The min-max regret of the LpNML is equal for all choices of $y$. If we consider a different predictive distribution, it should assign a smaller probability to at least one of the outcomes. If the true label is one of those outcomes, it will result in a greater regret. 
\end{proof}

Intuitively, the LpNML~\eqref{eq:lpnml} assigns a probability for a potential test label as follows: (i) Add the test sample to the training set with an arbitrary label $y'$, (ii) find the ERM solution with the defined luckiness function of this new set $\hat{\theta}_{y'}$, and (iii) take the probability it gives to the assumed label weighted by the luckiness function $p_{\hat{\theta}_{y'}}(y'|x) w(\thetaytag)$. Follow (i)-(iii)  for every possible test label value and normalize to get a valid predictive distribution. In the next section, we analytically derive the LpNML for ridge regression.

\subsection{LpNML for ridge regression}
We formulate the luckiness function as follows:
\begin{equation} \label{eq:luckiness_ridge}
w(\theta) = \exp\bigg\{-\frac{\lambda}{2\sigma^2}\norm{\theta}^2 \bigg\}.
\end{equation}
The genie learnable parameters using this luckiness function is the solution of the following minimization objective
\begin{equation}
\thetay  =
\argmin_\theta \bigg[ \sum_{n=1}^N \left(y_n - \theta^\top x_n \right)^2 + \left(y - \theta^\top x \right)^2 + \lambda \norm{\theta}^2\bigg].
\end{equation}
The solution is ridge regression which we express with the recursive least squares formulation~\cite{hayes19969}
\begin{equation} \label{eq:genie_rls}
\thetay  = \thetalamb + \frac{\Plamb x}{\Klamb}  \left(y - \thetalamb^\top x\right),
\end{equation}
where 
\begin{equation}
\Plamb \triangleq \left(X_N^\top X_N + \lambda I \right)^{-1}, \quad  \Klamb \triangleq 1+ x^\top \Plamb x.
\end{equation}

The goal is to analytically derive the LpNML predictive distribution \eqref{eq:lpnml} with the luckiness function of \eqref{eq:luckiness_ridge}. For this, we first have to determine the genie predictive distribution of the true test label.

\begin{lemma}
The genie predictive distribution weighted by the luckiness function of \eqref{eq:luckiness_ridge} is 
\begin{equation} \label{eq:lemma1}
\begin{split}
& \probthetay w(\thetay) = \\ &
\quad
\frac{c}{\sqrt{2\pi\sigma^2}}
\exp\left\{- \frac{1}{2 \sigmalpnml} \left(y-\thetalamb^\top x + \mulpnml \right)^2 
\right\}
\end{split}
\end{equation}
where 
\begin{equation} \label{eq:lpnml_mu_sigma}
\begin{split}
& \mulpnml \triangleq \frac{\lambda \Klamb \thetalamb^\top \Plamb x}{1+\lambda x^\top \Plamb^2 x}
, \quad
\sigmalpnml \triangleq \frac{\sigma^2 \Klamb^2}{1+\lambda x^\top \Plamb^2 x},
\\ & \qquad
c \triangleq 
\exp\bigg\{\frac{1}{2\sigma^2}\bigg[
\frac{\big(\lambda \thetalamb^\top \Plamb x\big)^2}{1+\lambda x^\top \Plamb^2 x}
- 
\lambda \norm{\thetalamb}^2 
\bigg]
\bigg\}.
\end{split}
\end{equation}
\end{lemma}
\begin{proof}
The genie probability assignment weighted by the luckiness function is
\begin{equation}
\begin{split}
&\probthetay w(\thetay) =
\\ &
\frac{1}{\sqrt{2\pi\sigma^2}}
\exp\left\{
-\frac{1}{2\sigma^2}
\left(y - \thetay^\top x \right)^2  
\right\}
\exp\left\{
-\frac{\lambda}{2\sigma^2} \norm{\thetay}^2
\right\}.
\end{split}
\end{equation}
Decomposing it with the recursive least squares formulation provides the result. A detailed proof is given in the appendix.
\end{proof}
$\mulpnml$ is the deviation of the genie prediction from the Ridge ERM prediction. 
If $\lambda=0$, the deviation is 0 and the genie prediction is equal to the ERM solution of~\eqref{eq:linear_regression_erm}.
The genie predictive distribution is not valid since $\int \probthetay dy>1$.
Next, we derive the LpNML by normalizing the genie predictive distribution weighted by the luckiness function.
\begin{theorem}
With the luckiness function of~\eqref{eq:luckiness_ridge}, the LpNML predictive distribution is
\begin{equation}
q_{\textit{\tiny{LpNML}}}(y|x) = \mathcal{N}\left(\thetalamb^\top x - \mulpnml, \sigmalpnml \right).
\end{equation}
\end{theorem}
\begin{proof}
Following \eqref{eq:lpnml}, to get the normalization factor we integrate the genie predictive distribution over all possible test label values. Utilizing lemma 1:
\begin{equation} \label{eq:nf_derivation}
\begin{split}
& K_\textit{\tiny{LpNML}} = \int_{- \infty}^{\infty} \probthetaytag w(\thetaytag) dy' 
\\ &=
\int_{-\infty}^{\infty} \frac{c}{\sqrt{2\pi\sigma^2}}
\exp\bigg\{-\frac{\big(y'-\thetalamb^\top x + \mulpnml \big)^2}{2 \sigmalpnml}  
\bigg\} dy'
\\ &=
\frac{c}{\sqrt{2\pi \sigma^2}} \sqrt{2\pi \sigmalpnml} = c \sqrt{\frac{\sigmalpnml}{\sigma^2}}.
\end{split}
\end{equation}
Dividing \eqref{eq:lemma1} by the normalization factor \eqref{eq:nf_derivation}, the LpNML predictive distribution is obtained.
\end{proof}
The LpNML prediction deviates from the Ridge ERM prediction by $\mulpnml$. Both $\mulpnml$ and the LpNML variance $\sigmalpnml$ depend on the test data $x$. 
When $\lambda=0$, the LpNML reduces to the pNML solution for under-parameterized linear regression that was derived by~\citet{bibas2019new}
\begin{equation}
    q_\textit{\tiny{pNML}} = \mathcal{N}\left(\thetahat^\top x, \sigma^2\left[1 + x\left(X_N^\top X_N \right)^{-1} x\right]^2\right).
\end{equation}
In the next section, we show that if the test data fall within the subspace corresponding to the small eigenvalues of the design matrix $X_N$, the deviation from Ridge ERM $\mulpnml$ is large which shifts the LpNML prediction toward 0.


\subsection{The learnable subspace} \label{sec:learnable_subpsace}
Focusing on over-parameterized linear regression, we analyze the dissimilarity between the LpNML and the Ridge ERM predictions:
\begin{equation}
\mulpnml 
=
\frac{\lambda \Klamb  \thetalamb^\top \Plamb x}{1+\lambda x^\top \Plamb^2 x}   
=
\frac{\lambda \left(1+x^\top \Plamb x\right) \thetalamb^\top \Plamb x}{1+\lambda x^\top \Plamb^2 x}.
\end{equation}
Let $u_m$ and $h_m$ be the $m$-th eigenvector and eigenvalue of the design matrix such that
\begin{equation}
X_N^\top X_N = \sum_{m=1}^M  h_m^2 u_m u_m^\top,
\end{equation}
for over-parameterized linear regression ($M > N$)
\begin{equation}
\begin{split}
x^\top P_\lambda x &=
\sum_{m=1}^N \frac{\left(u_m^\top x\right)^2}{h_m^2 + \lambda}
+
\sum_{m=N+1}^M \frac{\left(u_m^\top x\right)^2}{\lambda},
\\ 
x^\top P_\lambda^2 x &=
\sum_{m=1}^N \frac{\left(u_m^\top x\right)^2}{\left(h_m^2 + \lambda\right)^2}
+
\sum_{m=N+1}^M \frac{\left(u_m^\top x\right)^2}{\lambda^2}.
\end{split}
\end{equation}
We analyze two cases: A case where $x$ falls in the \textit{largest eigenvalue subspace} and a case where $x$ lies in the \textit{smallest eigenvalue subspace}.

\paragraph{Largest eigenvalue subspace.}
For a test feature $x$ that lies in the subspace that is spanned by the eigenvectors that are associated with the large eigenvalues of the design matrix $x^\top \Plamb x \ll 1$. The deviation from the Ridge ERM is
\begin{equation} \label{eq:lpnml_mu_parallel}
\begin{split}
& \mulpnml
=
\frac{\lambda \left(1+x^\top \Plamb x\right) \thetalamb^\top \Plamb x}{1+\lambda x^\top \Plamb^2 x} 
\approx
\lambda \thetalamb^\top \Plamb x 
\\ & \quad =
\lambda Y_N^\top  X_N \Plamb^2 x 
=
\lambda Y_N^\top X_N \sum_{m=1}^N \frac{u_m u_m^\top x}{\left(h_m^2 + \lambda \right)^2}
\ll
1.
\end{split}
\end{equation}
The LpNML prediction is similar to the Ridge ERM.
The LpNML variance in this case is
\begin{equation} \label{eq:lpnml_var_parallel}
\sigmalpnml = \frac{\sigma^2 \Klamb^2}{1 + \lambda x^\top \Plamb^2 x}
=
\frac{\sigma^2 \left(1 + x^\top \Plamb x \right)^2}{1 + \lambda x^\top \Plamb^2 x } \approx \sigma^2.
\end{equation}
There is a high level of confidence in the prediction since this is the smallest possible variance.

\paragraph{Smallest eigenvalue subspace.}
For a test vector that lies in the subspace that is spanned by the eigenvectors corresponding to the smallest eigenvalue of the regularized empirical correlation matrix of the training data
\begin{equation}
x^\top \Plamb x =  \sum_{m=N+1}^M \frac{\left(u_m^\top x\right)^2}{\lambda^2} = \frac{1}{\lambda} ||x||^2.   
\end{equation}
For a small regularization term $\frac{1}{\lambda} ||x||^2 \gg 1$, the deviation from the Ridge ERM prediction is
\begin{equation} \label{eq:lpnml_mu_orthogonal}
\begin{split}
\mulpnml
&=
\frac{\lambda \left(1+ \frac{||x||^2}{\lambda} \right) \thetalamb^\top \frac{x}{\lambda}}{1+\lambda \frac{||x||^2}{\lambda^2}}  
\approx
\frac{\lambda \frac{||x||^2}{\lambda} \thetalamb^\top \frac{x}{\lambda}}{\lambda \frac{||x||^2}{\lambda^2}}   
= 
\thetalamb^\top  x.
\end{split}
\end{equation}
The correction term equals the Ridge ERM prediction thus the LpNML prediction is shifted to 0. 
The LpNML variance in this situation is
\begin{equation} \label{eq:lpnml_var_orthogonal}
\sigmalpnml =
\sigma^2
\frac{\left(1 + \frac{||x||^2}{\lambda}\right)^2}{1 + \lambda \frac{||x||^2}{\lambda^2} } 
=
\sigma^2 \left(1 + \frac{||x||^2}{\lambda}\right).
\end{equation}
Compared to \eqref{eq:lpnml_var_parallel}, the variance is large which reflects high uncertainty in the prediction.

In~\secref{sec:experiments}, we empirically show the behavior of the LpNML variance on a synthetic set and demonstrate that the LpNML deviation from the Ridge ERM improves its performance for real-world datasets.

\section{Experiments} \label{sec:experiments}
We demonstrate the LpNML prediction behavior for fitting a polynomial function to synthetic data and for fitting over-parameterized linear regression to a high-dimensional synthetic dataset. In addition, we show that the LpNML outperforms the Ridge ERM for real-world PMLB datasets and attains state-of-the-art performance for the distribution-shift benchmark

\subsection{Polynomial fitting to synthetic data} \label{sec:exp_polinomial}
We sampled 6 training points $\left\{(t_n,y_n\right\}_{n=1}^6$ uniformly in the interval $[-1, 1]$.
and converted the data to features with a polynomial of degree 10 such that the design matrix is
\begin{equation} \label{eq:experiment_data_matrix}
X_N = 
\begin{bmatrix}
1 & t_1 & \dots & t_1^{10} \\
1 & t_2 & \dots & t_2^{10} \\
\vdots & \vdots & \ddots & \vdots \\
1 & t_6 & \dots & t_{6}^{10}
\end{bmatrix}.
\end{equation}
Based on this training set, we performed Bayesian and LpNML prediction for all $t$ values in the interval $[-1,1]$. We set the regularization $\lambda$ to $10^{-3}$.

\Figref{fig:pnml_syntetic_data} shows the Ridge ERM, Bayesian and LpNML predictions.
The Bayesian learner has the same prediction as to the Ridge ERM which is different from the LpNML: The LpNML prediction is closer to 0 in intervals that lack training data, e.g., $t \leq 0.75$ and $0.9 \leq t$.

The confidence intervals are shown in \figref{fig:pnml_syntetic_data} with lighter colors. Both learners have large confidence intervals at the figure's edges and for $0.1 \leq t \leq 0.8$, where the training points are scarce. The LpNML has much larger confidence intervals at the interval edges than the Bayesian learner.

\begin{figure}[t]
\centering
\includegraphics[width=1.0\columnwidth]{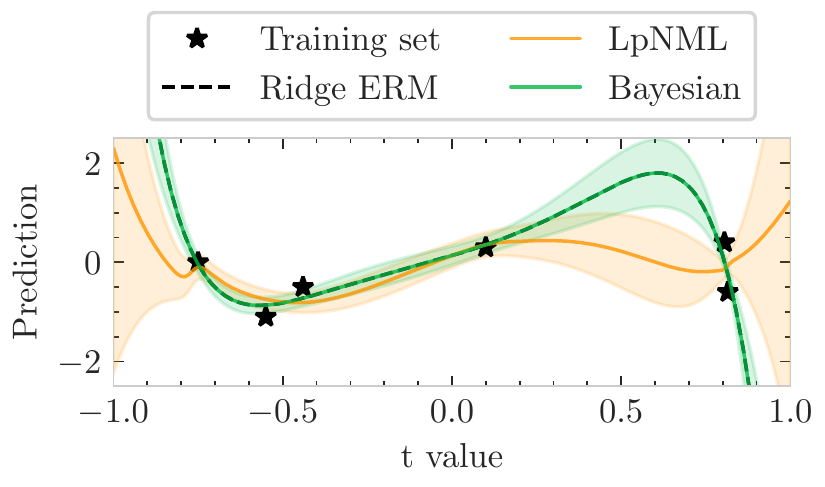}
\vspace{-0.5cm}
\caption{\textbf{Polynomial fitting to synthetic data.} 
The Ridge ERM, Bayesian and LpNML predictions fitted to the training data (in black stars) along with their confidence intervals. 
The LpNML largely deviates from the Ridge ERM in areas where the training data do not exist.
See~\secref{sec:exp_polinomial}.}
\label{fig:pnml_syntetic_data}
\end{figure}

\subsection{Parallel and orthogonal subspace} \label{sec:exp_lpnml_prediction}
We illustrate the LpNML behavior for test samples that lie in the largest eigenvalue subspace and smallest eigenvalue subspace as analyzed in~\secref{sec:learnable_subpsace}.

We created a synthetic dataset with $N=40$ training samples and $M=100$ features. We set the regularization term $\lambda$ to $10^{-9}$.
We randomly sampled $10,000$ test samples for each of the following two scenarios: The test sample data reside in the subspace that is spanned by the eigenvectors of the design matrix corresponding to the eigenvalues that are (a) greater than 0, or (b) equal 0.

\Figref{fig:lpnml_parallel} shows the histogram of Ridge ERM and LpNML predictions for test samples that reside in scenario (a).
The LpNML prediction for these test samples equals the Ridge ERM prediction which verifies equation~\eqref{eq:lpnml_mu_parallel} result.

\Figref{fig:lpnml_orthogonal} presents Ridge ERM and LpNML prediction histograms of samples from scenario (b).
The pNML predicts 0 while the Ridge ERM prediction varies between $-60$ and $100$, which is 2 order of magnitude larger than the prediction of scenario (a). The LpNML prediction of 0 is aligned with equation~\eqref{eq:lpnml_mu_orthogonal}: 
To avoid a large log-loss, the LpNML shifts the Ridge ERM prediction of test samples that differ from the training data to 0.

\begin{table*}[tbh!]
\centering
\caption{\textbf{Leave-one-out test performance for PMLB sets.} Ridge ERM and LpNML test MSE and the log-loss for Ridge ERM, Bayesian, and LpNML learners for real-world sets.
The evaluation was conducted using a training set size equal to the number of features ($M=N$). $\textcolor{green}{\blacktriangledown}$ and $\textcolor{pink}{\blacktriangle}$ indicate performance enhancement and degradation respectively. See~\secref{sec:exp_real_sets}.
}
\small
\begin{tabular}{lc|cc|ccc}
\toprule
               \hspace{0.7cm} Set name &  M & \thead{Ridge ERM \\ MSE} &  \thead{ LpNML \\ MSE} &       \thead{Ridge ERM \\ log-loss} &\thead{Bayesian \\ log-loss} & \thead{LpNML \\ log-loss} \\
\midrule
1199\_BNG\_echoMonths & 9 & 1.76 ± 0.55 & 1.41 ± 0.03$\textcolor{green}{\blacktriangledown20.0\%}$ & 173 ± 276 & 8.86 ± 9.83 & 2.64 ± 0.57$\textcolor{green}{\blacktriangledown6.22}$ \\
1089\_USCrime & 13 & 0.93 ± 0.02 & 0.86 ± 0.02$\textcolor{green}{\blacktriangledown8.30\%}$ & 1.76 ± 0.09 & 1.53 ± 0.04 & 1.40 ± 0.02$\textcolor{green}{\blacktriangledown0.14}$ \\
294\_satellite\_image & 36 & 0.90 ± 0.02 & 0.83 ± 0.02$\textcolor{green}{\blacktriangledown8.25\%}$ & 1.34 ± 0.01 & 1.34 ± 0.01 & 1.29 ± 0.01$\textcolor{green}{\blacktriangledown0.05}$ \\
banana & 2 & 2.84 ± 0.51 & 2.61 ± 0.35$\textcolor{green}{\blacktriangledown8.11\%}$ & 418 ± 101 & 296 ± 40.5 & 290 ± 40.5$\textcolor{green}{\blacktriangledown5.51}$ \\
195\_auto\_price & 15 & 1.28 ± 0.08 & 1.20 ± 0.08$\textcolor{green}{\blacktriangledown6.42\%}$ & 2.32 ± 0.28 & 1.71 ± 0.08 & 1.59 ± 0.15$\textcolor{green}{\blacktriangledown0.12}$ \\
695\_chatfield\_4 & 12 & 1.20 ± 0.07 & 1.13 ± 0.06$\textcolor{green}{\blacktriangledown5.67\%}$ & 1.73 ± 0.11 & 1.58 ± 0.08 & 1.44 ± 0.05$\textcolor{green}{\blacktriangledown0.14}$ \\
503\_wind & 14 & 1.13 ± 0.03 & 1.07 ± 0.03$\textcolor{green}{\blacktriangledown4.78\%}$ & 1.55 ± 0.03 & 1.51 ± 0.02 & 1.53 ± 0.12$\textcolor{pink}{\blacktriangle0.02}$ \\
560\_bodyfat & 14 & 0.59 ± 0.02 & 0.56 ± 0.02$\textcolor{green}{\blacktriangledown4.07\%}$ & 2.15 ± 0.35 & 1.32 ± 0.14 & 1.45 ± 0.23$\textcolor{pink}{\blacktriangle0.13}$ \\
659\_sleuth\_ex1714 & 7 & 1.69 ± 0.09 & 1.64 ± 0.09$\textcolor{green}{\blacktriangledown3.29\%}$ & 6.25 ± 1.89 & 3.00 ± 0.34 & 2.44 ± 0.19$\textcolor{green}{\blacktriangledown0.56}$ \\
344\_mv & 10 & 1.42 ± 0.09 & 1.38 ± 0.08$\textcolor{green}{\blacktriangledown2.89\%}$ & 7.71 ± 1.54 & 5.09 ± 1.36 & 4.64 ± 1.42$\textcolor{green}{\blacktriangledown0.45}$ \\
229\_pwLinear & 10 & 1.38 ± 0.04 & 1.35 ± 0.04$\textcolor{green}{\blacktriangledown2.25\%}$ & 3.32 ± 0.38 & 2.36 ± 0.18 & 2.23 ± 0.19$\textcolor{green}{\blacktriangledown0.13}$ \\
1027\_ESL & 4 & 1.72 ± 0.09 & 1.69 ± 0.08$\textcolor{green}{\blacktriangledown1.92\%}$ & 31.4 ± 5.82 & 24.7 ± 4.53 & 23.1 ± 4.49$\textcolor{green}{\blacktriangledown1.57}$ \\
653\_fri\_c0\_250\_25 & 25 & 1.16 ± 0.03 & 1.14 ± 0.03$\textcolor{green}{\blacktriangledown1.81\%}$ & 1.57 ± 0.05 & 1.53 ± 0.02 & 1.51 ± 0.02$\textcolor{green}{\blacktriangledown0.01}$ \\
230\_machine\_cpu & 6 & 13.1 ± 2.77 & 12.9 ± 2.72$\textcolor{green}{\blacktriangledown1.76\%}$ & 77.7 ± 79.8 & 8.31 ± 1.88 & 7.36 ± 1.97$\textcolor{green}{\blacktriangledown0.94}$ \\
1203\_BNG\_pwLinear & 10 & 1.41 ± 0.04 & 1.39 ± 0.04$\textcolor{green}{\blacktriangledown1.61\%}$ & 3.18 ± 0.49 & 2.33 ± 0.22 & 2.21 ± 0.23$\textcolor{green}{\blacktriangledown0.12}$ \\
561\_cpu & 7 & 13.2 ± 5.72 & 13.0 ± 5.69$\textcolor{green}{\blacktriangledown1.39\%}$ & 318 ± 266 & 12.4 ± 4.95 & 7.52 ± 3.06$\textcolor{green}{\blacktriangledown4.89}$ \\
564\_fried & 10 & 1.35 ± 0.03 & 1.34 ± 0.03$\textcolor{green}{\blacktriangledown1.33\%}$ & 4.07 ± 0.61 & 2.60 ± 0.28 & 2.31 ± 0.23$\textcolor{green}{\blacktriangledown0.29}$ \\
633\_fri\_c0\_500\_25 & 25 & 1.13 ± 0.02 & 1.12 ± 0.02$\textcolor{green}{\blacktriangledown1.24\%}$ & 1.54 ± 0.02 & 1.51 ± 0.01 & 1.51 ± 0.02\hphantom{00000} \\
598\_fri\_c0\_1000\_25 & 25 & 1.14 ± 0.02 & 1.12 ± 0.02$\textcolor{green}{\blacktriangledown1.20\%}$ & 1.57 ± 0.03 & 1.52 ± 0.02 & 1.51 ± 0.02$\textcolor{green}{\blacktriangledown0.01}$ \\
706\_sleuth\_case1202 & 6 & 2.50 ± 0.16 & 2.47 ± 0.16$\textcolor{green}{\blacktriangledown1.19\%}$ & 26.7 ± 9.12 & 11.1 ± 3.64 & 9.52 ± 3.61$\textcolor{green}{\blacktriangledown1.57}$ \\
651\_fri\_c0\_100\_25 & 25 & 1.10 ± 0.02 & 1.08 ± 0.02$\textcolor{green}{\blacktriangledown1.18\%}$ & 1.58 ± 0.04 & 1.52 ± 0.02 & 1.50 ± 0.02$\textcolor{green}{\blacktriangledown0.02}$ \\
635\_fri\_c0\_250\_10 & 10 & 1.41 ± 0.04 & 1.39 ± 0.04$\textcolor{green}{\blacktriangledown1.11\%}$ & 3.66 ± 0.86 & 2.42 ± 0.28 & 2.17 ± 0.20$\textcolor{green}{\blacktriangledown0.26}$ \\
656\_fri\_c1\_100\_5 & 5 & 2.59 ± 0.10 & 2.56 ± 0.10$\textcolor{green}{\blacktriangledown1.09\%}$ & 31.8 ± 7.88 & 19.1 ± 5.60 & 15.2 ± 4.72$\textcolor{green}{\blacktriangledown3.94}$ \\
1096\_FacultySalaries & 4 & 3.17 ± 0.35 & 3.14 ± 0.31$\textcolor{green}{\blacktriangledown1.06\%}$ & 107 ± 72.8 & 48.6 ± 20.9 & 13.6 ± 4.23$\textcolor{green}{\blacktriangledown35.0}$ \\
595\_fri\_c0\_1000\_10 & 10 & 1.34 ± 0.03 & 1.33 ± 0.03$\textcolor{green}{\blacktriangledown1.01\%}$ & 3.38 ± 0.42 & 2.41 ± 0.21 & 2.13 ± 0.16$\textcolor{green}{\blacktriangledown0.28}$ \\
1193\_BNG\_lowbwt & 9 & 1.48 ± 0.04 & 1.46 ± 0.04$\textcolor{green}{\blacktriangledown0.91\%}$ & 4.12 ± 0.94 & 3.09 ± 0.78 & 3.04 ± 0.80$\textcolor{green}{\blacktriangledown0.05}$ \\
650\_fri\_c0\_500\_50 & 50 & 1.06 ± 0.02 & 1.05 ± 0.02$\textcolor{green}{\blacktriangledown0.88\%}$ & 1.46 ± 0.01 & 1.45 ± 0.01 & 1.45 ± 0.01\hphantom{00000} \\
666\_rmftsa\_ladata & 10 & 2.73 ± 0.24 & 2.71 ± 0.24$\textcolor{green}{\blacktriangledown0.87\%}$ & 4.63 ± 2.25 & 2.87 ± 0.75 & 2.48 ± 0.32$\textcolor{green}{\blacktriangledown0.39}$ \\
1028\_SWD & 10 & 1.40 ± 0.03 & 1.38 ± 0.03$\textcolor{green}{\blacktriangledown0.86\%}$ & 5.56 ± 2.52 & 3.85 ± 1.42 & 2.89 ± 0.66$\textcolor{green}{\blacktriangledown0.96}$ \\
192\_vineyard & 2 & 11.5 ± 1.56 & 11.4 ± 1.45$\textcolor{green}{\blacktriangledown0.85\%}$ & 722 ± 212 & 698 ± 212 & 669 ± 211$\textcolor{green}{\blacktriangledown29.2}$ \\
225\_puma8NH & 8 & 1.48 ± 0.04 & 1.47 ± 0.03$\textcolor{green}{\blacktriangledown0.82\%}$ & 5.32 ± 0.77 & 3.17 ± 0.33 & 2.82 ± 0.39$\textcolor{green}{\blacktriangledown0.34}$ \\
634\_fri\_c2\_100\_10 & 10 & 1.47 ± 0.04 & 1.46 ± 0.04$\textcolor{green}{\blacktriangledown0.78\%}$ & 3.68 ± 1.54 & 2.39 ± 0.35 & 1.93 ± 0.10$\textcolor{green}{\blacktriangledown0.46}$ \\
657\_fri\_c2\_250\_10 & 10 & 1.47 ± 0.05 & 1.46 ± 0.05$\textcolor{green}{\blacktriangledown0.67\%}$ & 4.59 ± 2.42 & 2.25 ± 0.21 & 2.20 ± 0.25$\textcolor{green}{\blacktriangledown0.05}$ \\
603\_fri\_c0\_250\_50 & 50 & 1.06 ± 0.02 & 1.06 ± 0.02$\textcolor{green}{\blacktriangledown0.49\%}$ & 1.45 ± 0.01 & 1.45 ± 0.01 & 1.45 ± 0.01\hphantom{00000} \\
624\_fri\_c0\_100\_5 & 5 & 2.11 ± 0.11 & 2.10 ± 0.10$\textcolor{green}{\blacktriangledown0.39\%}$ & 31.4 ± 11.8 & 16.7 ± 10.5 & 8.71 ± 3.49$\textcolor{green}{\blacktriangledown8.03}$ \\
606\_fri\_c2\_1000\_10 & 10 & 1.52 ± 0.05 & 1.52 ± 0.05$\textcolor{green}{\blacktriangledown0.35\%}$ & 4.53 ± 2.00 & 3.95 ± 1.95 & 3.29 ± 1.88$\textcolor{green}{\blacktriangledown0.66}$ \\
579\_fri\_c0\_250\_5 & 5 & 2.00 ± 0.07 & 2.00 ± 0.07$\textcolor{green}{\blacktriangledown0.22\%}$ & 22.0 ± 4.73 & 11.7 ± 3.53 & 6.85 ± 1.61$\textcolor{green}{\blacktriangledown4.83}$ \\
648\_fri\_c1\_250\_50 & 50 & 1.08 ± 0.02 & 1.08 ± 0.02$\textcolor{green}{\blacktriangledown0.20\%}$ & 1.47 ± 0.02 & 1.46 ± 0.01 & 1.46 ± 0.01\hphantom{00000} \\
1191\_BNG\_pbc & 18 & 1.23 ± 0.03 & 1.23 ± 0.03$\textcolor{green}{\blacktriangledown0.19\%}$ & 2.10 ± 0.53 & 2.02 ± 0.52 & 1.86 ± 0.46$\textcolor{green}{\blacktriangledown0.16}$ \\
618\_fri\_c3\_1000\_50 & 50 & 1.10 ± 0.03 & 1.09 ± 0.03$\textcolor{green}{\blacktriangledown0.17\%}$ & 1.47 ± 0.01 & 1.47 ± 0.01 & 1.47 ± 0.01\hphantom{00000} \\
631\_fri\_c1\_500\_5 & 5 & 2.38 ± 0.09 & 2.38 ± 0.09$\textcolor{green}{\blacktriangledown0.14\%}$ & 38.2 ± 10.0 & 24.0 ± 6.32 & 24.7 ± 8.23$\textcolor{pink}{\blacktriangle0.67}$ \\
583\_fri\_c1\_1000\_50 & 50 & 1.06 ± 0.02 & 1.06 ± 0.02$\textcolor{green}{\blacktriangledown0.14\%}$ & 1.45 ± 0.01 & 1.45 ± 0.01 & 1.45 ± 0.01\hphantom{00000} \\
586\_fri\_c3\_1000\_25 & 25 & 1.18 ± 0.03 & 1.18 ± 0.03$\textcolor{green}{\blacktriangledown0.13\%}$ & 1.69 ± 0.23 & 1.64 ± 0.19 & 1.52 ± 0.02$\textcolor{green}{\blacktriangledown0.12}$ \\
542\_pollution & 15 & 1.32 ± 0.05 & 1.32 ± 0.05$\textcolor{green}{\blacktriangledown0.09\%}$ & 1.85 ± 0.15 & 1.60 ± 0.02 & 1.57 ± 0.02$\textcolor{green}{\blacktriangledown0.04}$ \\
687\_sleuth\_ex1605 & 5 & 2.17 ± 0.09 & 2.17 ± 0.09$\textcolor{green}{\blacktriangledown0.08\%}$ & 15.1 ± 3.06 & 7.66 ± 1.52 & 6.06 ± 1.31$\textcolor{green}{\blacktriangledown1.60}$ \\
645\_fri\_c3\_500\_50 & 50 & 1.09 ± 0.03 & 1.09 ± 0.03$\textcolor{green}{\blacktriangledown0.05\%}$ & 1.47 ± 0.02 & 1.46 ± 0.02 & 1.46 ± 0.02\hphantom{00000} \\
623\_fri\_c4\_1000\_10 & 10 & 1.55 ± 0.05 & 1.54 ± 0.05$\textcolor{green}{\blacktriangledown0.04\%}$ & 2.64 ± 0.52 & 2.01 ± 0.12 & 1.99 ± 0.13$\textcolor{green}{\blacktriangledown0.02}$ \\
622\_fri\_c2\_1000\_50 & 50 & 1.08 ± 0.02 & 1.08 ± 0.02$\textcolor{green}{\blacktriangledown0.04\%}$ & 1.46 ± 0.01 & 1.46 ± 0.01 & 1.46 ± 0.01\hphantom{00000} \\
658\_fri\_c3\_250\_25 & 25 & 1.18 ± 0.03 & 1.18 ± 0.03\hphantom{0000000} & 1.54 ± 0.02 & 1.52 ± 0.02 & 1.52 ± 0.02\hphantom{00000} \\
201\_pol & 48 & 1.07 ± 0.02 & 1.07 ± 0.02\hphantom{0000000} & 1.45 ± 0.01 & 1.45 ± 0.01 & 1.45 ± 0.01\hphantom{00000} \\
\bottomrule
\end{tabular}

\label{tab:real_data}
\end{table*}

\subsection{Leave on out real data performance} \label{sec:exp_real_sets}
We evaluated the LpNML for 50 real-world datasets from the PMLB repository~\cite{Olson2017PMLB}.
To tune the ridge parameter $\lambda$ and the variance $\sigma^2$, we executed the leave-one-out procedure: We constructed with $N$ samples $N$ sets, each set was divided into $N-1$ training samples and a single validation sample for which we optimized $\lambda$ and $\sigma^2$ of Ridge ERM, Bayesian, and LpNML learners.
The average of the $N$ values of $\lambda$ and $\sigma^2$ were used to predict the test labels.
We repeated this procedure for different train-test splits to compute the 95\% confidence intervals.
The MSE reduction is measured in percentage
\begin{equation}
100 \times \left( 1 - \frac{MSE_\textit{LpNML}(x,y)}{MSE_\textit{Ridge-ERM}(x,y)}\right)
\end{equation}
and the log-loss reduction is measured with subtraction
\begin{equation}
\ell(q_\textit{\tiny{Bayesian}},x,y) - \ell(q_\textit{\tiny{LpNML}},x,y).
\end{equation} 

\Tableref{tab:real_data} shows the test MSE and test log-loss. The LpNML outperforms the Ridge ERM for 48 of 50 sets.
The mean and median MSE reductions are 2.03\%  and  0.96\%, respectively.
The largest MSE reduction is 20.0\% for the 1199\_BNG\_echoMonths set.
For the log-loss metric, the LpNML outperforms the other learners in 37 out of 50 sets by a mean value of 2.17 and median value of 0.13, while degrading the log-loss of only 3 sets.

Overall, the LpNML has a smaller regularization term: For the 1199\_BNG\_echoMonths dataset, the LpNML has 1.45 lower $\lambda$. This may explain the better performance: For the interpolation region, the small $\lambda$ of the LpNML means a better fit to the test sample.
For the extrapolation region: Although the LpMNL has a smaller $\lambda$, the LpNML prediction is shifted toward 0, and therefore the LpNML test MSE is smaller than the Ridge ERM MSE.

\begin{table*}[hbt!]
\centering
\caption{\textbf{Distribution-shift benchmark.} Test RMSE for real-world datasets with a distribution shift between training and test, except the Triazines set that was randomly split. The LpNML outperforms recent leading methods for 4 of 5 datasets. LpNML is the second-best method for the Fertility set. For more details see~\secref{sec:dist_shift}.
}
\label{table:dist_shift}
\begin{tabular}{lccccc}
\toprule
TriazinesMethod                    & Wine         & Parkinson                & Fire             & Fertility                & Triazines \\ 
\toprule
OLS                                 & 1.012$\pm$0.016       & 12.792$\pm$0.149                  & 82.715$\pm$35.514         & 0.399$\pm$0.066                   & 0.172$\pm$0.037   \\
Ridge ERM                           & 0.994$\pm$0.015       & 12.527$\pm$0.145                  & 82.346$\pm$35.595         & 0.399$\pm$0.066                   & 0.147$\pm$0.028  \\
\citet{chapelle2000transductive}    & 0.841$\pm$0.001       & 12.253$\pm$0.002                  & 82.066$\pm$ 2.567          & 0.409$\pm$0.013                  & 0.173$\pm$0.001   \\
\citet{cortes2007transductive}      & 0.834$\pm$0.015       & 12.333$\pm$0.145                  & 81.947$\pm$35.834         & \textbf{0.385$\pm$0.076}          & 0.151$\pm$0.024  \\ 
\citet{alquier2012transductive}     & 0.981$\pm$0.015       & 12.253$\pm$0.136                  & 82.066$\pm$36.032         & 0.409$\pm$0.072                   & 0.148$\pm$0.024  \\
\citet{tripuraneni2020single}       & 0.770$\pm$0.014       & 12.089$\pm$0.137                  & 81.979$\pm$35.787         & 0.398$\pm$0.065                   & 0.151$\pm$0.024   \\
\citet{dwivedi2021revisiting}       & 0.929$\pm$0.015       & 12.693$\pm$0.147                  & 82.634$\pm$35.533         & 0.407$\pm$0.071                   & 0.166$\pm$0.021  \\
LpNML (ours)                        & \textbf{0.732$\pm$0.014}       & \textbf{12.027$\pm$0.142}    & \textbf{81.918$\pm$35.923}        & 0.398$\pm$0.067       & \textbf{0.147$\pm$0.024}   \\ 
\bottomrule
\end{tabular}
\end{table*}

\begin{figure*}[ht]
    \centering
    \includegraphics[width=\textwidth]{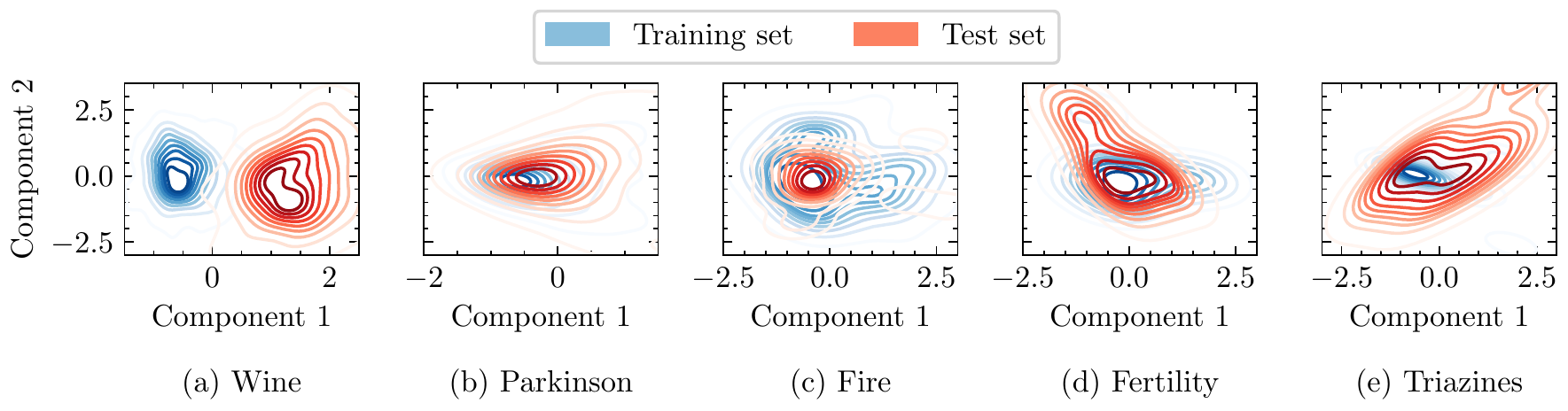}
     \vspace{-0.6cm}
    \caption{\textbf{PCA for the distribution-shift benchmark}. In the Wine dataset, where LpNML outperforms the leading method most, the train-test split is most visually different. For more information see~\secref{sec:dist_shift}.}
    \label{fig:dist_shift}
\end{figure*}

\subsection{Distribution-shift benchmark} \label{sec:dist_shift}
We followed the benchmark that was proposed by~\citet{tripuraneni2020single}: Four datasets from the UCI repository~\cite{Dua:2019} were chosen and split such that the test data contain a distribution shift from the training. 
The fifth dataset (Triazines) does not include a distribution shift. 
The train-test split was performed randomly. A detailed explanation of the train-test split is provided in the appendix. This benchmark was evaluated using the RMSE metric and the hyperparameters were optimized using the leave-one-out procedure.

The LpNML attains a smaller RMSE for four sets. The largest improvement is for the Wine set for which the LpNML reduces the RMSE of~\citet{tripuraneni2020single} by 4.93\%.
For the Fertility set, the LpNML is the second-best following the method of~\citet{cortes2007transductive}, which has 3.27\% smaller RMSE.
For the Triazines set that does not contain a distribution shift, the LpNML performs the same as the Ridge ERM and outperforms the other methods.

\Figref{fig:dist_shift} shows the principal component analysis (PCA) with 2 components of the benchmark sets.
For the Wine set in \figref{fig:dist_shift}a, the difference between the train and test data is the most visually seen. For this set, the LpNML has the largest RMSE reduction over the Ridge ERM: an RMSE reduction of 26.36\%.
For the Parkinson, Fire, Fertility, and Triazines datasets, the LpNML reduces the Ridge ERM RMSE by 3.99\%, 0.52\%, 0.25\%, 0.0\% respectively, which is correlated to the degree to which the train-test splits are visually separated.

\FloatBarrier
\section{Conclusions} \label{sec:conclusion}
In this work, we introduced the LpNML by incorporating a luckiness function to the min-max regret objective.
For ridge regression, where we defined the luckiness function as the Gaussian prior, we have shown that the LpNML prediction is shifted toward 0 for a test vector that resides in the subspace that is associated with the small eigenvalues of the design matrix.
For real-world datasets, the LpNML attains up to 20\% better test MSE than Ridge ERM and for the distribution-shift benchmark, the LpNML reduces the error of recent leading methods by up to 4.93\%.

We believe that our approach can be valuable to fields that use linear regression and require high-precision prediction.
For future work, our LpNML framework can be extended with more luckiness functions such as $\ell_1$ by defining the luckiness function to be the Laplace prior. 



\bibliography{main.bib}
\bibliographystyle{icml2022}

\newpage
\appendix
\onecolumn

\section{Lemma 1 proof}
The genie probability assignment weighted by the luckiness function is
\begin{equation} \label{eq:genie_prob_derivation}
\begin{split}
&\probthetay w(\thetay) 
=
\frac{1}{\sqrt{2\pi\sigma^2}}
\exp\left\{
-\frac{1}{2\sigma^2}
\left(y - \thetay^\top x \right)^2  
\right\}
\exp\left\{
-\frac{\lambda}{2\sigma^2} \norm{\thetay}^2
\right\}.
\end{split}
\end{equation}
The first exponent argument using the recursive least squares formulation~\eqref{eq:genie_rls}:
\begin{equation}
\begin{split}
\big(y - \thetay^\top x \big)^2  
&= 
\left[y
- \thetalamb^\top x
-
\frac{1}{\Klamb} x^\top \Plamb \big(y - \thetalamb^\top x\big)x \right]^2 =
\left[ \bigg(1 -\frac{x^\top \Plamb x}{\Klamb}\bigg) y 
-
\left(
1
-
\frac{x^\top \Plamb x}{\Klamb}
\right)
\thetalamb^\top x \right]^2
=
\frac{1}{\Klamb^2}\big(y - \thetalamb^\top x \big)^2.
\end{split}
\end{equation}
Substituting it back to~\eqref{eq:genie_prob_derivation} and using again the recursive least squares
\begin{equation}
\begin{split}
& \probthetay  w(\thetay) = 
\frac{1}{\sqrt{2\pi\sigma^2}}  \exp\bigg\{ 
-\frac{\big(y - \thetalamb^\top x\big)^2}{2 \sigma^2 \Klamb^2} 
-\frac{\lambda}{2\sigma^2} \bnorm{\thetalamb + \frac{\Plamb x}{\Klamb}  \big(y - \thetalamb^\top x\big)}^2
\bigg\}.
\end{split}
\end{equation}
Deriving the exponential argument:
\begin{equation}
\begin{split}
\frac{1}{\Klamb^2}
\big(y - \thetalamb^\top x\big)^2
+\lambda \bnorm{ \thetalamb + \frac{\Plamb x}{\Klamb}  \big(y - \thetalamb^\top x\big)}^2
&= 
\frac{1 + \lambda x^\top \Plamb^2 x}{\Klamb^2} \big(y - \thetalamb^\top x\big)^2
+ 
2\frac{\lambda \thetalamb^\top \Plamb x}{\Klamb} \big(y - \thetalamb^\top x\big)
+ 
\lambda \norm{\thetalamb}^2  
\\ &=
\frac{1+\lambda x^\top \Plamb^2 x}{\Klamb^2}
\big(y - \thetalamb^\top x +  \frac{\lambda \Klamb \thetalamb^\top \Plamb x}{1+\lambda x^\top \Plamb^2 x} \big)^2 
-
\frac{\big(\lambda \thetalamb^\top \Plamb x\big)^2}{1+\lambda x^\top \Plamb^2 x}
+ \lambda \norm{\thetalamb}^2 .
\end{split}
\end{equation}
Substitute in \eqref{eq:genie_prob_derivation} proves the lemma.

\begin{figure}[bth]
\centering
\begin{subfigure}{0.49\textwidth}
\centering
    \includegraphics[width=\textwidth]{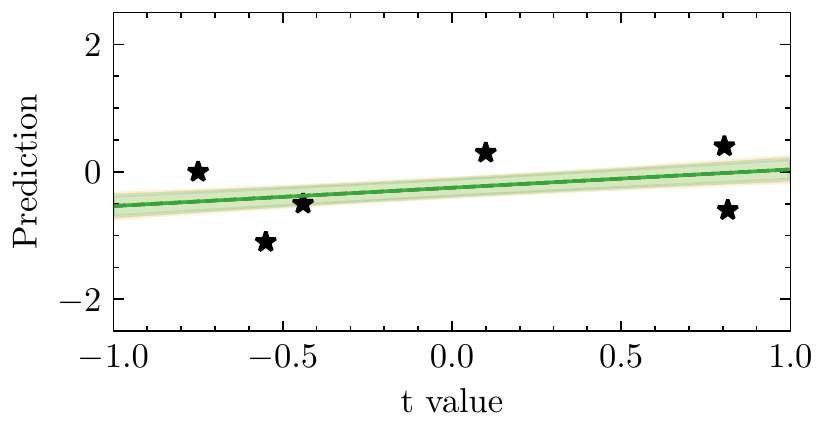}
    \vspace{-0.5cm}
    \caption{$M=2$}
    \label{fig:m_2}
\end{subfigure}
\begin{subfigure}{0.49\textwidth}
\centering
    \includegraphics[width=\textwidth]{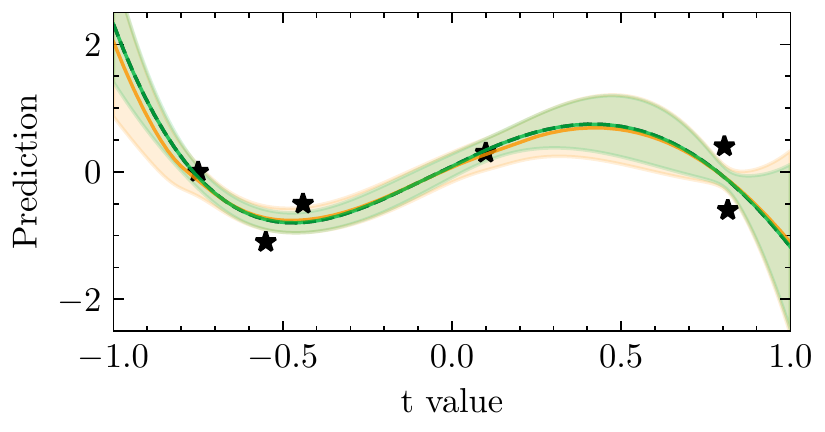}
    \vspace{-0.5cm}
    \caption{$M=5$}
    \label{fig:m_5}
\end{subfigure}
\begin{subfigure}{0.49\textwidth}
\centering
    \includegraphics[width=\textwidth]{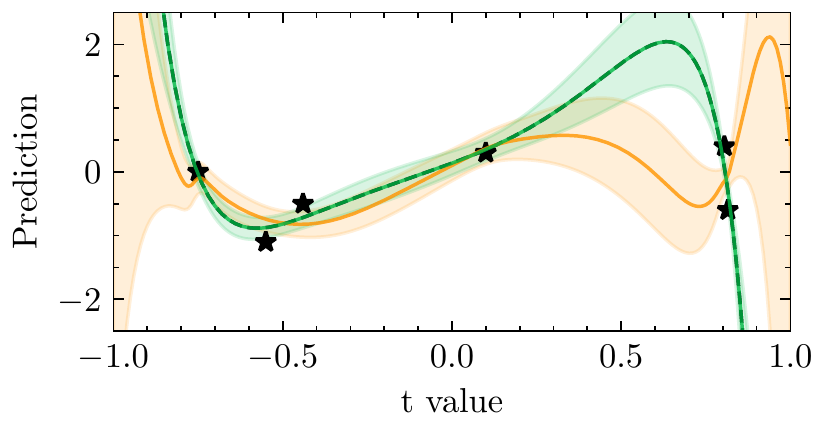}
    \vspace{-0.5cm}
    \caption{$M=20$}
    \label{fig:m_20}
\end{subfigure}
\begin{subfigure}{0.49\textwidth}
\centering
    \includegraphics[width=\textwidth]{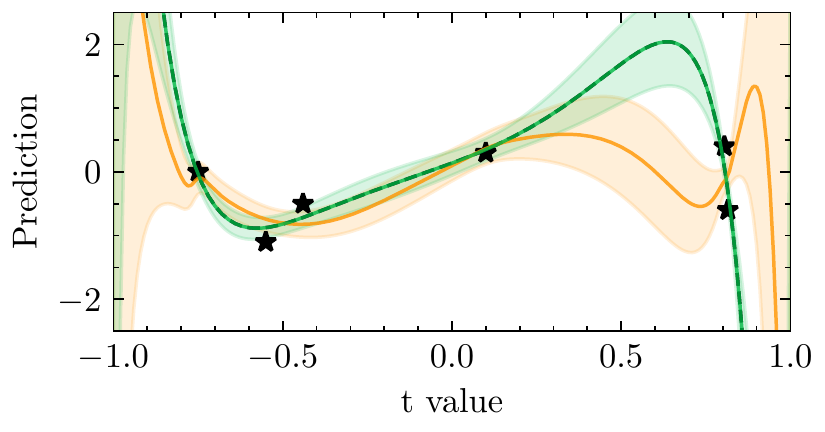}
    \vspace{-0.5cm}
    \caption{$M=100$}
    \label{fig:m_100}
\end{subfigure}
\caption{The LpNML (orange line) and Bayesian (green line) learners with different model degrees.}
\end{figure}

\section{Additional synthetic data results}
We present additional results for synthetic data as was described in \secref{sec:exp_polinomial}. The results here include more model degrees $M$. \Figref{fig:m_2} and \figref{fig:m_5} show the prediction for the under-parameterized case where $M < N$: The LpNML and Bayesian have a similar prediction and both have small variance.
For $M=20$ and $M=100$ as demonstrated in \figref{fig:m_20} and \figref{fig:m_100} respectively, in the absence of training data the LpNML prediction is closer to 0 than the Bayesian prediction.

\section{Distribution shift datasets}
In this section, we elaborate on the train-test split that was performed to the distribution-shift benchmark by~\cite{tripuraneni2020single}.

The prediction task for the Wine dataset is to predict acidity levels: The training data comprised only red wines with a test set contains only white wines.
In the Parkinsons dataset, the task is to predict a jitter index. This set was split into train and test based on the age feature of patients: Age less than 60 for the train set and greater than 60 for the test set.
For the Fertility dataset, the task is to predict the fertility of a sample. The train contains subjects who are younger than 36 and the test set contains subjects older than 36.
Finally, for the Fires dataset, where the task is to predict the burned area of forest fires that occurred in Portugal during a roughly year-long period, the split was done into train/test based on the month feature of the fire: Those occurring before September for a train set and those after September for the test set.
The Triazines dataset does not include a distribution shift and was randomly split.

In~\Tableref{tab:dist_shift_metadata}, we include further information on these datasets and \figref{fig:pca_3_4} presents the PCA third and forth components.

\begin{figure}[bth]
    \centering
    \includegraphics[width=\textwidth]{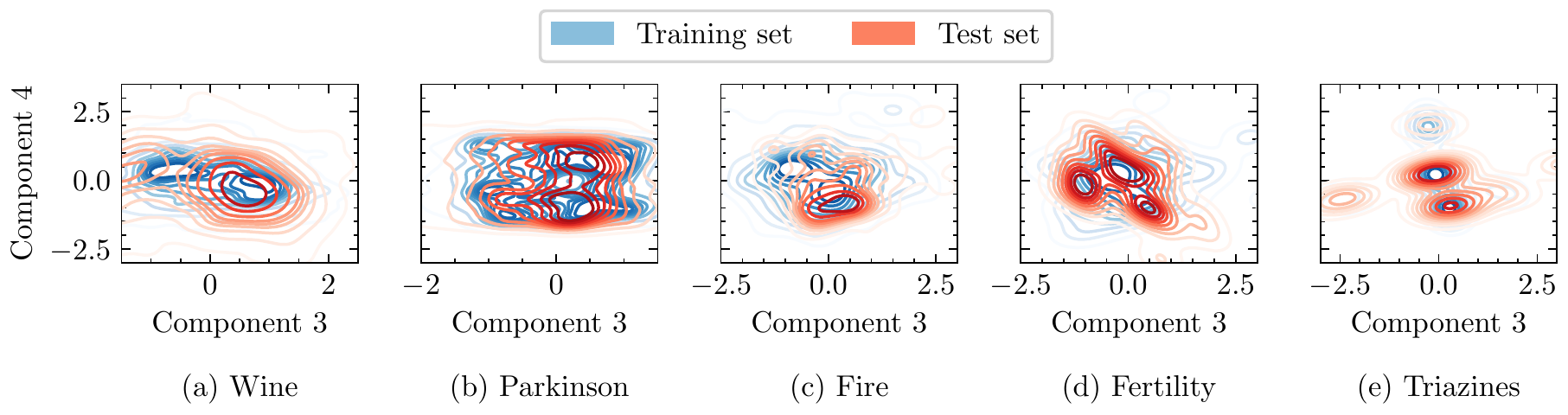}
    \vspace{-0.6cm}
    \caption{PCA for the distribution-shift benchmark: components 3 and 4.}
    \label{fig:pca_3_4}
\end{figure}

\begin{table}[tbh]
\centering
\caption{Distribution shift dataset characteristics.}
\label{tab:dist_shift_metadata}
\begin{tabular}{cccc}
\toprule
Dataset & M & Training set size & Test set size \\
\toprule
Wine        &   8   & 69    & 31            \\
Parkinson   & 10    & 320   & 197          \\
Fire        & 17    & 1877  & 3998         \\
Fertility   & 11    & 4898  & 1599         \\
Triazines   & 60    & 139   & 47          \\
\bottomrule
\end{tabular}
\end{table}

\end{document}